\documentclass{article}
\pdfoutput = 0

\usepackage{nips15submit_e}
\usepackage{times}
\usepackage{url}
\usepackage[numbers]{natbib}
\usepackage{graphicx}

\usepackage{algorithm}
\usepackage{algorithmic}

\usepackage[cmex10]{amsmath}
\usepackage{amssymb}

\definecolor{MyDarkBlue}{rgb}{0.15,0.25,0.85}
\definecolor{MyDarkGreen}{rgb}{0.1,0.6,0.1}
\definecolor{MyDarkBoh}{rgb}{0.1,0.6,0.8}

\newcommand{\argmin}{\mathop{\mathrm{arg\,min}}}

\def\Pp{\bf \mathcal{P}}
\def\C{{\bf C}}

\def\L{{\bf L}}

\def\s{{\bf s}}

\def\tr{\text{Tr}}
\newcommand{\x}{{\bf x}}
\newcommand{\xs}{{\bf X}_s}
\newcommand{\xt}{{\bf X}_t}
\newcommand{\z}{{\bf z}}

\newcommand{\y}{{\bf y}}
\renewcommand{\xi}{{\bf x}_i}
\newcommand{\xsi}{{\bf x}^s_i}
\newcommand{\xti}{{\bf x}^t_i}

\newcommand{\G}{{\bf {x}}}
\newcommand{\Ga}{\boldsymbol{\gamma}}
\newcommand{\Gzeroreg}{\boldsymbol{ \gamma}^\lambda_{0}}

\newtheorem{theorem}{Theorem}[section]
\newtheorem{lemma}[theorem]{Lemma}
\newtheorem{proposition}[theorem]{Proposition}

\newenvironment{proof}[1][Proof]{\begin{trivlist}
\item[\hskip \labelsep {\bfseries #1}]}{\end{trivlist}}
\newenvironment{definition}[1][Definition]{\begin{trivlist}
\item[\hskip \labelsep {\bfseries #1}]}{\end{trivlist}}

\newcommand{\qed}{\nobreak \ifvmode \relax \else
      \ifdim\lastskip<1.5em \hskip-\lastskip
      \hskip1.5em plus0em minus0.5em \fi \nobreak
      \vrule height0.75em width0.5em depth0.25em\fi}

\def\argmin{\mathop{\mathrm{argmin}}}

\def\Pp{\bf \mathcal{P}}
\def\C{{\bf C}}

\def\R{\mathbb{R}}

\def\C{\mathbf{C}}

\def\x{\mathbf{x}}

\def\Z{\mathbf{Z}}

\def\y{\mathbf{y}}

\def\one{\mathbf{1}}

\newtheorem{proposition*}{\textbf{Proposition}}
\newtheorem{theorem*}{\textbf{Theorem}}
\newtheorem{remark*}{\textbf{Remark}}
\newtheorem{definition*}{\textbf{Definition}}

\nipsfinalcopy

\title{Generalized conditional gradient: analysis of convergence and applications}

\author{
Alain Rakotomamonjy\\
Universit\'e de Rouen\\ France\\
\texttt{alain.rakoto@insa-rouen.fr} \\
\And
R{\'e}mi Flamary \\
Lagrange, CNRS, UNS, OCA\\ France\\
\texttt{remi.flamary@unice.fr} \\
\AND
Nicolas Courty \\
Universit\'e de Bretagne du Sud, IRISA\\ France\\
\texttt{ncourty@irisa.fr}
}

\begin{document}

\maketitle
\begin{abstract}
The objectives of this technical report is to provide additional results
on the generalized conditional gradient methods introduced by Bredies et al. \cite{bredies2005equivalence}.
Indeed, when the objective function is smooth, we provide a novel certificate of optimality and we show that the algorithm
has a  linear convergence rate. Applications of this algorithm are also discussed.

\end{abstract}

\section{Generalized conditional gradient}
\label{sec:optim-fram-algor}

\label{sec:da}

We are interested in the problem of minimizing under constraints a composite function such
as
\begin{equation}\label{eq:prob_fws}
  \min_{\G \in \Pp}\quad F(\G)=f(\G) +g(\G),
\end{equation}
where both $f(\cdot)$ and $g(\cdot)$ are convex and differentiable
functions and $ \Pp$ is compact set of $\R^n$.
One might want to benefit from this composite structure during
the optimization procedure. For instance, if we have an efficient solver
for optimizing
\begin{equation}
 \min_{\G \in \Pp}\quad \langle {\nabla f}, \G \rangle +g(\G)\label{eq:linear_prox}
\end{equation}
it is of prime interest  to use this solver in the optimization scheme instead of
linearizing the whole objective function as one would do with a conditional
gradient algorithm \cite{bertsekas1999nonlinear}.

The resulting approach is defined in Algorithm \ref{algo:fws}, denoted CGS in the remainder (the S standing for Splitting).
{Conceptually, this algorithm lies in-between the original optimization
problem and the conditional gradient. Indeed, if we do
not consider any linearization, the step 3 of the algorithm is equivalent
to solving the original problem and one iterate will suffice for convergence.
If we use a full linearization as in the conditional gradient approach, step 3
is equivalent to solving a rough approximation of the original problem.
By linearizing only a part of the objective function, we optimize a better
approximation of that function, as opposed to a full linearization as in the conditional gradient approach.  This leads to a provably better certificate
of optimality than the one of the conditional gradient algorithm \cite{jaggi2013revisiting}. Also, if an efficient solver of the
partially linearized problem is available, this algorithm is of strong interest. This is notably the case in computation of regularized optimal transport problems~\cite{courty2015},
or Elastic-net regularization~\cite{zou2005regularization}, see the application Section for details.

Note that this partial linearization idea  has already been introduced by Bredies et al. \cite{bredies2005equivalence} for solving problem \ref{eq:prob_fws}.  Their theoretical results related
to the resulting algorithm apply  when $f(\cdot)$ is  differentiable, $g(\cdot)$ convex $f$ and $g$ satisfy some others mild conditions like coercivity. These
results state that the generalized conditional gradient algorithm is a descent method and that any limit point of the algorithm is a stationary point of $f + g$.

In what follows we provide some results when $f$ and $g$ are differentiable.
Some of these results provide novel insights on the generalized gradient
algorithms (relation between optimality and minimizer of the search direction, convergence rate, and optimality certificate; while some are redundant to those
proposed by Bredies (convergence).

\begin{algorithm}[t]
\caption{Generalized Gradient Splitting (CGS)}
\label{algo:fws}
\begin{algorithmic}[1]
\STATE Initialize $k=0$ and $\G^0\in\mathcal{P}$
\REPEAT
\STATE Solve problem     $\s^k=\argmin_{\G \in \Pp} \quad\left < \nabla f(\G^k),\G\right>+ g(\G)$
\STATE Find the optimal step $\alpha^k$ with $\Delta\G=\s^k-\G^k$
    \begin{equation*}
      \alpha^k=\argmin_{0\leq \alpha \leq 1}\quad f(\G^k+\alpha \Delta\G)+g(\G^k+\alpha \Delta\G)
    \end{equation*}
or choose $\alpha^k$ so that it satisfies the Armijo rule.
\STATE $\G^{k+1}\leftarrow \G^k+\alpha^k \Delta\G$, {s}et $k\leftarrow k+1$
\UNTIL{Convergence}
\end{algorithmic}
\end{algorithm}

\section{Convergence analysis}

Before discussing convergence of the algorithm, we first reformulated
its step 3 so as to make its properties more accessible
and its convergence analysis more amenable.

The reformulation we propose is
\begin{equation}\label{eq:newstep}
 \s^k=\argmin_{\s \in \Pp}\quad \langle \nabla f(\G^k), \s - \G^k \rangle
 + g(\s) - g(\G^k)
 \end{equation}
and it is easy to note that the problem in
line 3 of Algorithm \ref{algo:fws},  is equivalent to this
one and leads to the same solution.

The above formulation allows us to derive a property that
highligths the relation between  problems \ref{eq:prob_fws} and \ref{eq:newstep}.

\begin{proposition}
  $\x^\star$ is a minimizer of  problem (\ref{eq:prob_fws}) if and
only if
\begin{equation}
 \G^\star=\argmin_{\s \in \Pp}\quad \langle \nabla f(\G^\star), \s - \G^\star \rangle
 + g(\s) - g(\G^\star)
\end{equation}
\end{proposition}

\begin{proof}
  The proof relies on optimality conditions of constrained convex optimization
problem. Indeed, for a convex and differentiable $f$ and $g$, $\G^\star$ is
solution of  problem (\ref{eq:prob_fws}) if and only if  \cite{bertsekas_convexanalysis}
\begin{equation}\label{eq:optcond}
-\nabla f(\G^\star) - \nabla g(\G^\star) \in N_{\Pp}(\G^\star)
\end{equation}
where $ N_{\Pp}(\G)$ is the normal cone of $\Pp$ at $\G$.
In a same way, a minimizer $\s^\star$ of problem (\ref{eq:newstep}) at $\G^k$ can also be characterized as
\begin{equation}\label{eq:optcondlinear}
- \nabla f(\G^k)
 - \nabla g(\s^\star) \in N_{\Pp}(\s^\star)
\end{equation}
Now suppose that $\G^\star$ is a minimizer of problem (\ref{eq:prob_fws}),
it is easy to see that if we choose $\G^k = \G^\star$ then because
$\G^\star$ satisfies Equation \ref{eq:optcond}, Equation
\ref {eq:optcondlinear} also holds.
Conversely, if $\G^\star$ is a minimizer of problem (\ref{eq:newstep}) at $\G^\star$
then $\G^\star$ also satisfies Equation \ref{eq:optcond}.
\end{proof}

\subsection{Intermediate results and gap certificate}

We  prove several lemmas and we exhibit a gap  certificate that provide a
bound on the difference of the  objective value along the iterations
to the optimal objective value.

As one may remark our algorithm is very similar to a conditional
gradient algorithm \emph{a.k.a} Frank-Wolfe algorithm. As such, our
proof of convergence of the algorithm will follow similar lines
as those used by Bertsekas. Our convergence results is based
on the following proposition and definition given in \cite{bertsekas1999nonlinear}. \\

\begin{proposition}{    \cite{bertsekas1999nonlinear}}
\label{prop:bertsekas}
  Let $\{\G^k\}$ be a sequence generated by the feasible direction
method $\G^{k+1} = \G^{k} + \alpha_k \Delta \G$ with $\Delta \G^k = \s^k - \G^k$.
Assume that $\{\Delta \G^k\}$ is gradient related and that
$\alpha^k$ is chosen by the limited minimization or the Armijo rule, then
every limit point of $\{\G^k\}$ is a stationary point.
\end{proposition}

\begin{definition} A sequence $\Delta \G^k$ is said
to be gradient related to the sequence $\G^k$ if  for any subsequence of $\{\G^k\}_{k \in K}$ that converges to
    a non-stationary point, the corresponding subsequence $\{\Delta
    \G^k\}_{k \in K}$ is bounded and satisfies
$$
\limsup_{k \rightarrow \infty, k \in K} \nabla F(\G^k)^\top\Delta \G^k
< 0
$$
\end{definition}

Basically, this property says that if a subsequence converges
to a non-stationary point, then at the limit point the feasible direction
defined by $\Delta \G$ is still a descent direction.  Before proving that
the sequence defined by  $\{\Delta \G^k\}$ is gradient related, we prove
useful lemmas. \\

\begin{lemma}
For any $\G^k \in \Pp$, each  $\Delta \G^k =  \s^k- \G^k$ defines a feasible descent direction.
\end{lemma}
\begin{proof}
  By definition, $\s^k$ belongs to the convex set $\Pp$. Hence, for any
$\alpha^k \in [0,1]$,  $\G^{k+1}$ defines a feasible point. Hence  $\Delta \G^k$
is a feasible direction.

Now let us show that it also defines a descent direction. By definition of
the minimizer $\s^k$, we
have for all $\s \in \Pp$
\begin{align*}
  \langle\nabla f(\G^k), \s^k - \G^k \rangle + g(\s^k) - g(\G^k) \leq &
  \langle \nabla f(\G^k), \s - \G^k \rangle + g(\s) - g(\G^k)\label{eq:1}
\end{align*}
because the above inequality also holds for $\s = \G^k$, we have
\begin{equation}\label{eq:dd}
\langle\nabla f(\G^k), \s^k - \G^k \rangle + g(\s^k) - g(\G^k) \leq 0
\end{equation}
By convexity of $g(\cdot)$, we have
$$
g(\s^k) - g(\G^k) \geq \langle \nabla g(\G^k), \s^k- \G^k\rangle
$$
which, plugged in equation, \ref{eq:dd} leads to
$$
\langle\nabla f(\G^k) + \nabla g(\G^k), \s^k - \G^k \rangle   \leq 0
$$
and thus
$\langle\nabla F(\G^k), \s^k - \G^k \rangle   \leq 0$, which proves that $\Delta \G^k$
is a descent direction.
\end{proof}

The next lemma provides an interesting feature of our
algorithm. Indeed, the lemma states that the difference between the
optimal objective value and
the current objective value can be easily monitored.  \\
\begin{lemma}\label{lm:bound}
  For all $\G^k \in \Pp$, the following property holds
$$\min_{\s \in \Pp}\quad \big[\langle \nabla f(\G^k), \s - \G^k \rangle
+ g(\s) - g(\G^k)\big] \leq   F(\G^\star) - F(\G^k) \leq 0
$$
where $\G^\star$ is a minimizer of $F$.
In addition, if $\G^k$ does not belong to the set of minimizers
of $F(\cdot)$, then the second inequality is strict.
\end{lemma}

\begin{proof}
By convexity of $f$, we have
$$
f(\G^\star)- f(\G^k) \geq \nabla f(\G^k)^\top(\G^\star - \G^k)
$$
By adding $g(\G^\star)- g(\G^k)$ to both side of the inequality, we
obtain
$$
F(\G^\star)- F(\G^k) \geq \nabla f(\G^k)^\top(\G^\star - \G^k) +
g(\G^\star)- g(\G^k)
$$
and because $\G^\star$ is a minimizer of $F$, we also have
$ 0 \geq F(\G^\star)- F(\G^k)$
Hence, the following holds
$$ \langle \nabla f(\G^k), \G^\star - \G^k \rangle
+ g(\G^\star) - g(\G^k) \leq F(\G^\star) - F(\G^k) \leq 0
$$
and we also have
$$\min_{s \in \Pp}\quad \langle \nabla f(\G^k), \s - \G^k \rangle
+ g(\s) - g(\G^k) \leq F(\G^\star) - F(\G^k) \leq 0
$$
which concludes the first part of the proof. \\
Finally if $\G^k$ is not a minimizer of $F$, then we naturally have $0> F(\G^\star)- F(\G^k)$.
\end{proof}

\subsection{ Proof of convergence}
\label{sec:L}

Now that we have all the pieces of the proof, let us show the key ingredient.
\begin{lemma}The sequence $\{\Delta \G^k\}$ of our algorithm is gradient related.
\end{lemma}
\begin{proof}
 For showing that our direction sequence is gradient related, we have to
show that given a subsequence $\{\G^k\}_{k \in K}$ that converges
to a non-stationary point $\tilde \G$, the sequence
$\{\Delta \G^k\}_{k \in K}$ is bounded and that
$$
\limsup_{k \rightarrow \infty, k \in K} \nabla F(\G^k)^\top\Delta \G^k < 0
$$

Boundedness of the sequence naturally derives from the fact that
$\s^k \in \Pp$, $\G^k \in \Pp$ and the set $\Pp$ is bounded. \\

The second part of the proof starts by showing that
\begin{align}\nonumber
\langle \nabla F(\G^k), \s^k - \G^k \rangle &= \langle \nabla f(\G^k) + \nabla
g(\G), \s^k - \G^k \rangle
& \leq \langle \nabla f(\G^k), \s^k - \G^k \rangle
+ g(\s^k) - g(\G^k)\nonumber
\end{align}
where the last inequality is obtained owing to the convexity of $g$. Because
that inequality holds for the minimizer, it also holds for any vector $ s \in \Pp$ :
$$\langle \nabla F(\G^k), \s^k - \G^k \rangle
 \leq \langle \nabla f(\G^k), \s - \G^k \rangle
+ g(\s) - g(\G^k)
$$
Taking limit yields to
$$ \limsup_{k \rightarrow \infty, k \in K} \langle \nabla F(\G^k), \s^k - \G^k \rangle
 \leq \langle \nabla f(\tilde \G), \s - \tilde \G \rangle
+ g(\s) - g(\tilde \G)
$$
for all $\s \in \Pp$. As such, this inequality also holds
for the minimizer
$$ \limsup_{k \rightarrow \infty, k \in K} \langle\nabla F(\G^k), \s^k - \G^k \rangle
 \leq \min_{\s \in \Pp} \langle \nabla f(\tilde \G), \s - \tilde \G \rangle
+ g(\s) - g(\tilde \G)
$$
Now, since $\tilde \G$ is not stationary, it is not optimal and it does
not belong to the minimizer of $F$, hence according to the above
lemma \ref{lm:bound},
$$
\min_{\s \in \Pp} \quad\langle \nabla f(\tilde \G), \s - \tilde \G \rangle
+ g(\s) - g(\tilde \G) < 0
$$
which concludes the proof.
\end{proof}

This latter lemma proves that our direction sequence is gradient related, thus
proposition~\ref{prop:bertsekas} applies.

\subsection{ Rate of convergence}

We can show that the objective valut $F(\x_k)$ converges
towards $F(\x^\star)$ in a linear rate if we have some additional
smoothness condition of $F(\cdot)$. We can easily prove this
statement by following the steps proposed by Jaggi et al. \cite{jaggi2013revisiting} for
the conditional gradient algorithm.

We make the hypothesis that there exists a constant
$C_F$ so that for any $\x, \y \in \Pp$ and any $\alpha \in [0,1]$,
so that the inequality
$$
F((1-\alpha)\x + \alpha \y) \leq F(\x) + \alpha \nabla F(\x)^\top (\y-\x) + \frac{C_F}{2}\alpha^2
$$
holds.

Based on this inequality, for a sequence $\{\x_k\}$ obtained
from the generalized conditional gradient algorithm we have

\begin{align}\nonumber
 F(\x_{k+1}) - F(\x^\star)   & \leq F(\x_k) - F(\x^\star) + \alpha_k \nabla F(\x_k)^\top(s_k - \x_k)  +  \frac{C_F}{2}\alpha_k^2
\end{align}
Let us denote as $h(\x_{k})=F(\x_k) - F(\x^\star)$, now
by adding to both sides of the inequality $ \alpha_k[g(\s_k) - g(\x_k)]$
we have
\begin{eqnarray}\nonumber
  h(\x_{k+1}) + \alpha_k[g(\s_k) - g(\x_k)] & \hfill\\\nonumber
\leq  h(\x_k) +
\alpha_k [\nabla f(\x_k)^\top(\x^\star - \x_k) + g(\x^\star) - g(\x_k)]
+ \alpha_k \nabla g(\x_k)^\top(\s_k - \x_k) + \frac{C_F}{2}\alpha_k^2 \nonumber
\end{eqnarray}
where the second inequality comes from the definition of the search direction
$\s_k$. Now because, $f(\cdot)$ is convex we have
$ f(\x^\star) - f(\x_k) \geq \nabla f(\x_k)^\top (\x^\star - \x_k)$. Thus we have
\begin{eqnarray}\nonumber
  h(\x_{k+1}) + \alpha_k[g(\s_k) - g(\x_k)] & \leq  (1-\alpha_k) h(\x_k)  + \alpha_k \nabla g(\x_k)^\top(\s_k - \x_k) + \frac{C_F}{2}\alpha_k^2 \nonumber
\end{eqnarray}
Now, again, owing to the convexity of $g(\cdot)$, we have
$0 \geq -g(\s_k) + g(\x_k) \nabla g(\x_k)^\top(\s_k - \x_k)$. Using
this fact in the last above inequality leads to
\begin{equation}\label{eq:decrease}
  h(\x_{k+1})  \leq  (1-\alpha_k) h(\x_k)  + \frac{C_F}{2}\alpha_k^2
\end{equation}
Based on this result, we can now state the following
\begin{theorem}
For each $k \geq 1$, the iterates $\x_k$ of Algorithm $1$
satisfy
$$
F(\x_k)- F(\x^\star) \leq \frac{2 C_F}{k+2}
$$
\end{theorem}
\begin{proof}
The proof stands on Equation \ref{eq:decrease} and on the same induction as the one used by Jaggi et al \cite{jaggi2013revisiting}.
\end{proof}
Note that this convergence property would also hold if we choose the
step size as $\alpha_k = \frac{2}{k+2}$.

\subsection{Related works and discussion}
\label{sec:related-works}

Here, we propose to take advantage of the composite structure of
the objective function and of an efficient solver of that partially linearized
problem.}
{Note that a result similar to Lemma \ref{lm:bound},
  denoted as the surrogate duality gap in
  \cite{jaggi2013revisiting}, exists for the conditional gradient:
$$\min_{\s \in \Pp}\quad \langle \nabla F(\G^k), \s - \G^k \rangle \leq   F(\G^\star) - F(\G^k) \leq 0.
$$
Using the convexity of $g(\cdot)$ one can see that
$$\quad \langle \nabla F(\G^k), \s - \G^k \rangle
\leq \langle \nabla f(\G^k), \s - \G^k \rangle
+ g(\s) - g(\G^k).
$$
This means that the bound expressed in Lemma \ref{lm:bound} is at
least as good that the one provided by the classical CG. In addition
when $g(\cdot)$ is strictly convex, our bond is strictly better which
suggests that our approach provides a better control of the convergence along the iterations.
}

The approach that is  the most related to this method is probably the
conditional gradient for composite optimization proposed in
\cite{harchaoui2013conditional}. In their work, the authors show that
when the constraint can be expressed as a bound on a norm, it
 is can be more efficient to solve an equivalent regularized
version, \emph{i.e.} a composite problem. By solving the equivalent
problem, they can benefit from efficient computation for nuclear norm
regularization and Total Variation in images. The main
difference with our approach is that they focus on potentially
nondifferentiable norm based regularization for $g(\cdot)$ whereas our
algorithm can be used for any convex and differentiable $g(\cdot)$.

Finally our approach is also closely related to projected gradient
descent and its spectral variant \cite{birgin2000nonmonotone}. As
discussed more in detail in section
\ref{sec:learn-with-eslat}, when $g(\cdot)$ is an euclidean distance, the
solving problem \eqref{eq:linear_prox} boils down to a projection onto
the convex $\Pp$. In practice the method is more general since it can
be used for any problem as long as problem \eqref{eq:linear_prox} can
be efficiently solved.

\section{ Application to machine learning problems and numerical
  experiments}
\label{sec:real-life-appl}

In this section, we showcase two interesting applications of CGS on difficult machine learning
problems. First, we discuss the application of CGS on the regularized
optimal transport problem, and then we
consider a more general application framework of our conditional
gradient splitting algorithm to the widely used elastic-net regularization.

\subsection{Regularized optimal transport}
\label{sec:regul-optim-ransp}
Regularized optimal transport have recently been developed as an elegant
way to model several classical problems in machine learning and image
processing. For instance, it is used for color transfer in images, a
problem that consists in matching the colors of a source image to that of a
target one~\cite{ferradans13,rabin2014non}. It has also been
successfully applied to the problem of unsupervised domain adaptation
where one seeks  to adapt a classifier from a source domain (used for learning) to a
target domain (used for testing) \cite{courty2014domain} and recently
it has been
also considered as an efficient way to compute Wasserstein barycenters in
\cite{Cuturi14}.

Regularized optimal transport consists in searching an optimal transportation plan to
move one source distribution onto a target one, with particular conditions on the plan. In the particular case
where only discrete samples $\{\xsi\}_{1 \leq i \leq n_s}$ and $\{\xti\}_{1 \leq i \leq n_t}$ are available for respectively the source and target distributions,
we note the corresponding distributions as vectors $\mu^s\in\R^{+n_s}$
and $\mu^t\in\R^{+n_t}$. Usually, $\mu^s$ and
$\mu_t$ are seen as histograms since
both belong to the probability simplex. In the optimal transport problem, those distributions are embedded in metric
spaces, which allows to define a transport cost. In the discrete case, this metric is given as a matrix $\C$, for which each
pairwise term $C_{i,j}$ measures the cost of transporting the $i$th component of $\mu^s$ onto the $j$th component of $\mu^t$.
OT aims at finding a positive matrix $\Ga$ which can be seen as a joint probability distribution between the source and target,
with marginals  $\mu^s$ and $\mu_t$. It belongs to the set of doubly stochastic matrices or Birkhoff polytope.
The optimal transport plan is the one which minimizes the total cost of the transport  of $\mu^s$ onto $\mu^t$. The regularization
applies on the matrix $\Ga$, and aims at favoring particular conditioning of this plan.
The corresponding optimization problem reads
\begin{align}\label{eq:cuturi}
\Gzeroreg = &\argmin_{\Ga \in \Pp} \quad \left < \Ga, \C \right >_F
+\lambda \Omega(\Ga),\\
&\quad \text{s.t.}\quad \Ga \geq 0, \quad \Ga \one_{n_t} = \mu^s,\quad \Ga^\top \one_{n_s} = \mu^t\nonumber
\end{align}
where $\Omega(\cdot)$ is a regularization term. In the particular case
where  one considers an information theoretic measure on $\Ga$, namely the negentropy, this term can be written as $\Omega(\Ga)=\Omega_{\bf IT}(\Ga)=\sum_{i,j} \Ga(i,j) \log \Ga(i,j)$. \citet{Cuturi13}
proposed an extremely efficient algorithm , which uses the scaling
matrix approach of Sinkhorn-Knopp~\cite{knight08} to solve this problem.

Other types of regularizations can be considered. In \cite{courty2014domain}, authors use a group sparse
regularization term to prevent elements from different classes to be matched in the target domain. In~\cite{ferradans13},
$\Ga$ is regularized such that an estimated positions of the source
samples, transported in the target domain, are
consistently moved with respect to their initial spatial
distribution. It has been applied with success to color transfer between images
where pixels are seen as 3D samples. The same approach has been also
tested for domain adaptation and 3D shape matching in
\cite{flamary2014optlaplace}.
It can be seen a Laplacian regularization,
where the regularization term reads
\begin{equation}
 \Omega_{\bf Lap}(\Ga) = \lambda_s\tr({\xt}^\top\Ga^\top \L_s \Ga\xt)+\lambda_t\tr( {\xs}^\top\Ga \L_t \Ga^\top \xs ).\nonumber
\label{eq:lap}
\end{equation}
Here, $\L_s$ and $\L_t$ are the Laplacian matrices and $\xs$ and $\xt$
are the matrices of source and target samples positions. The two terms
both aim at preserving
the shape of the two distributions, with respective regularization parameters  $\lambda_s$ and $ \lambda_t$.

We consider in this paper a problem which combines the two regularizations:
\begin{align}
\Gzeroreg = &\argmin_{\Ga} \quad \left < \Ga, \C \right >_F
+\lambda_1 \Omega_{\bf IT}(\Ga) +\lambda_2 \Omega_{\bf Lap}(\Ga) ,\\
&\quad \text{s.t.}\quad \Ga \geq 0, \quad \Ga \one_{n_t} = \mu^s,\quad \Ga^\top \one_{n_s} = \mu^t\nonumber
\label{eq:regularizedtransport}
\end{align}
This problem is hard to solve, for several reasons: the presence of
the entropy related term  prevents  quadratic
programming strategies to be applied, and because the Laplacian
matrices are generally dense the objective function can be costly to
compute. {But this problem fits particularly well into the CGS
  framework. Indeed, we have a constrained optimization
problem where the objective function is smooth and we have at our disposal
the efficient algorithm of \cite{Cuturi13} that is able to solve a
partially linearized
version of the objective function under the same constraints.} In this context,
, we define $f(\Ga)=\quad \left < \Ga, \C \right >_F +  \lambda_2 \Omega_{\bf Lap}(\Ga)$
and $g(\Ga)=\lambda_1 \Omega_{\bf IT}(\Ga)$.
{According to these definitions, the problem of finding $\s_k$ boils
down to
\begin{align}
\s_k = &\argmin_{\Ga} \quad \left < \Ga, \C  + \lambda_2 \nabla
\Omega_{\bf Lap}(\Ga) \right >_F
+\lambda_1 \Omega_{\bf IT}(\Ga) ,\\\nonumber
&\quad \text{s.t.}\quad \s\geq 0, \quad \Ga \one_{n_t} = \mu^s,\quad \Ga^\top \one_{n_s} = \mu^t
\label{eq:stepcuturi}\end{align}} for which  an efficient solver exists as it is equivalent to
the problem addressed in Equation ~\ref{eq:cuturi} in the
particular case of the negentropy regularization~\cite{Cuturi13}.
Note that while $\Omega_{\bf IT}$ is not differentiable in 0,
the Sinkhorn-Knopp algorithm never leads to exact zero coefficients (See Fig;
  1 in \cite{Cuturi13}), hence $g(.)$ is differentiable for all
  iterations.

We study the performances of our approach on a simulated example
similar to the one illustrated in Figure 3.1 of
\cite{ferradans13}. Samples from the source and target domains are
generated with a noisy cluster structure. In order to keep this
structure the symmetric Laplacien regularization is constructed from the graph
of $10$ nearest neighbors in each domain. The regularization
parameters have been selected as to promote a transport of the graph
structure with a reasonable regularization ($\lambda_1=1.7\times
10^{-2}$ and $\lambda_2=10^3$).
The experiments are performed with the dimensions
$n_s=n_t=100$ and $n_s=n_t=500$, leading to a total number of
variables $n_s\times n_t$.
In the experiments we compare our method with the conditional
gradient algorithm on the exact same problem.
{In this case, fully linearizing the objective function leads to
a linear problem for finding $\s_k$.} Note that while other approaches based on
proximal methods have been proposed to solve this kind of optimization
problem \cite{papadakis14}, we do not think we can compare fairly since
they are not interior point methods and the intermediate iterations
might violate the constraints.

We report in Figure~\ref{fig:ot_reg} both objective values along iterations (left column) and along the overall computational
time (right column) for the two problems.
Two different implementations of linear programming solvers were used for comparisons: CVXopt~\cite{cvxopt} and MOSEK~\cite{mosek}.
In the second examples, the CVXopt experiments were not reported as it took too much time to complete the optimization. As one can see,
the CGS outperforms in all cases the CG approaches, both in  overall computation time, and also because it reaches better
objective values. This difference is also amplified when the size of the problem increases, and can be seen on the last row. {Note that the gain in computational time brought by our algorithm is about  an order of magnitude better than a CG algorithm using MOSEK.}

\begin{figure*}[t]
  \centering    \includegraphics[width=1\linewidth]{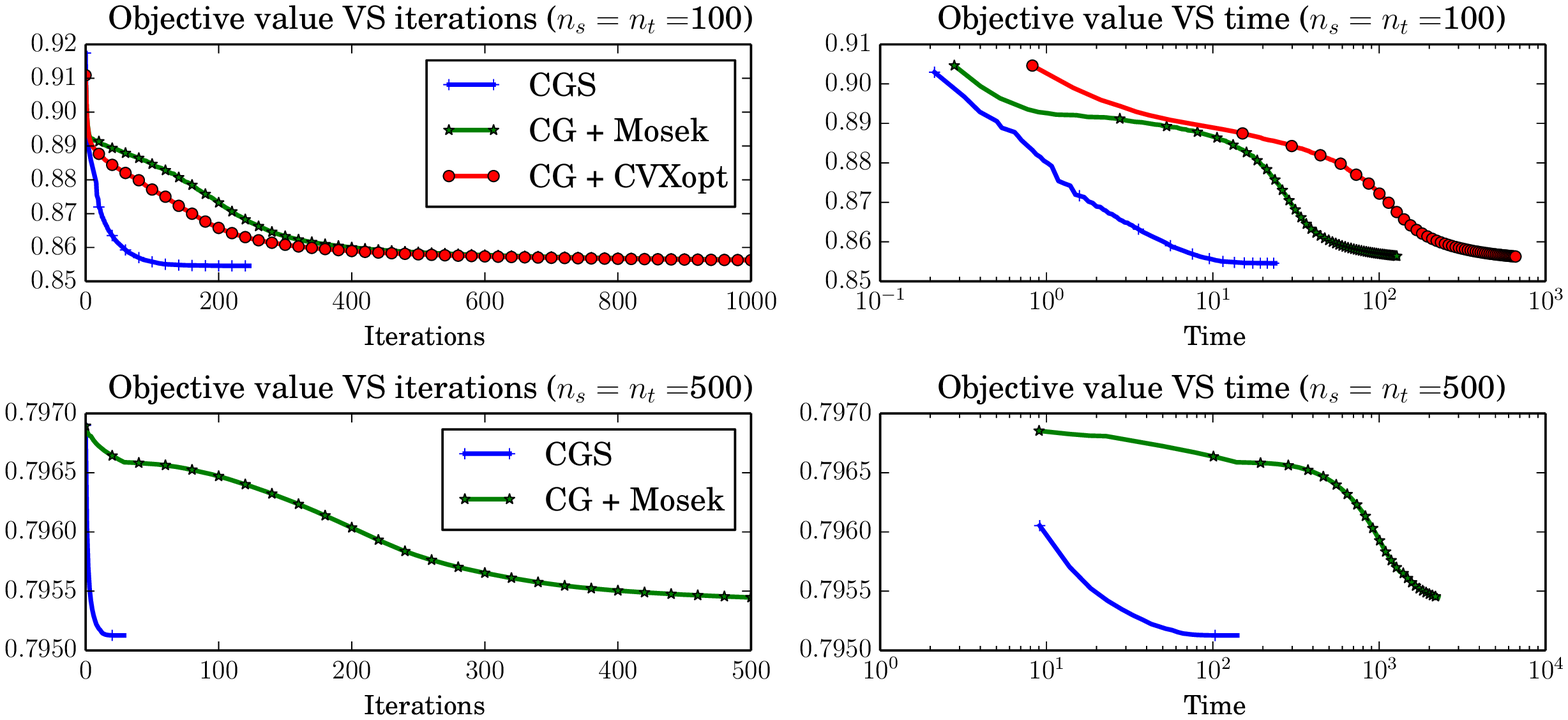}

  \caption{Objective value for a regularized optimal transport problem
    of size 100x100 (top) and 500x500 (bottom)
    along the iterations and along time for CGS and CG with different solvers.}
  \label{fig:ot_reg}
\end{figure*}

\subsection{Learning with elastic-net regularization}
\label{sec:learn-with-eslat}

Elastic-net regularization has been introduced by \cite{zou2005regularization} as a way to balance some undesirable behaviour of the $\ell_1$
penalty.   It has
been mainly motivated by the fact that it allows the selection of groups
of correlated features and yields consistent prediction and feature
selection \cite{de2009elastic}. This regularization is composed of the
weighted sum of
a $\ell_1$-norm and a squared $\ell_2$-norm regularization.
In this work, we want to show that
the $\ell_1$ norm-constrained version of the elastic-net problem can be
efficiently learned by our  conditional gradient splitting algorithm.
Denote as $\{\z_i,y_i\}_{ 1 \leq i \leq m}$ a set of observed data, with
$\z_i \in \R^n$ being a feature vector and $y_i \in \R$ the target samples.
The objective is to learn a linear model $\z^\top \x$ that predicts
the target $y$. For this purpose, we consider the following constrained
elastic-net problem
\begin{align}
\min_{\x \in \R^n}\quad&\ L(\y,\Z\x) + \lambda \x^\top\x \\
\text{s.t.} &\,\, \|\x\|_1 \leq \tau\label{eq:elasticnet}
\end{align}
where $\Z \in \R^{m \times n}$ is the matrix composed by stacking
rows of $\z_i^\top$, $\y$ is the vector of the $y_i$.
and  $L(\cdot,\cdot)$ is a differentiable loss function that measures
the discrepancy between each coordinate of $\y$ and $\Z\x$.
In the context
of our conditional gradient splitting algorithm, we will define  $f(\x)$ as
$L(\y,\Z\x)$ and $g(\x)= \lambda \x^\top\x$.
Accordingly, the step $3$ of the algorithm becomes
$$
 \s^k= \argmin_{\|\x\|_1 \leq \tau}\quad \x^\top\nabla f(\x^k)\ + \lambda \x^\top\x
$$
Interestingly, this problem can be easily solved as it can be shown
that the above problem is equivalent to
$$
 \s^k= \argmin_{\|\x\|_1 \leq \tau} \quad\Big\|\x - \big( -\frac{1}{2\lambda}\nabla f(\x^k) \big)\
\Big \|_2^2
$$
Hence, the feasible point $\s_k$  is the projection of the scaled negative
gradient onto the set defined by the constraint. As such, in this particular
case where $g(\x)$ is a quadratic term, our algorithm has the flavor of
a gradient projection method. The main difference resides in the fact that
in our CGS algorithm, it is the negative gradient that is projected
onto the constraint set, instead of the point resulting from a step along
the negative gradient. Formally, at each iteration, we thus have
\begin{equation}
  \label{eq:stepelastic}
  \x^{k+1} = (1-\alpha) \x^k + \alpha \Pi_{\|\x\|_1 \leq \tau}
\left(-\frac{1}{2\lambda} \nabla f(\x^k) \right)
\end{equation}
and $\x^{k+1}$ is a linear combination of the current iterate
and the projected scaled negative gradient.

This framework can be extended
to any constraint set, and we can expect the algorithm to be
efficient as long as projection of the set can be computed in a cheap
way. In addition the algorithm can be used for any convex and
differentiable data fitting term and thus it can be used also for
classification with squared hinge loss \cite{chapelle2007training} or
logistic regression \cite{krishnapuram2005sparse}. Note however that
in these latter cases, the optimal step $\alpha$ can not
be computed in a closed form as in a least-square loss context
\cite{ferradans13}.

In the following, we have illustrated   the behaviour of our
conditional gradient splitting
algorithm compared to classical projected gradient algorithms on
 toy and real-world classification problems using an elastic-net logistic
regression problem.
As such, we have considered the limited-memory projected quasi-newton
(PQN) method \cite{schmidt09} and the spectral projected gradient
(SPG) method \cite{birgin2000nonmonotone} both implemented by Schmidt. In our comparisons, we have also included the
original conditional gradient
algorithm as well as an heuristic conditional gradient splitting with
step $\alpha^k$ set as $\alpha^k=\frac{2}{k+2}$.
For all algorithms, we have used a monotone
armijo rule as a linesearch algorithm and  the stopping criterion is based on the fixed point property of a minimizer of
(\ref{eq:elasticnet}) (also used
by Schmidt)
$$ \Big \|\Pi_{\|\x\|_1 \leq \tau}(\x - \nabla F(\x)) - \x \Big \|_\infty  \leq \varepsilon$$
In our experiments, we have set $\varepsilon = 10^{-5}$ and we have also
set the maximal number of iterations to $10000$ for all algorithms.
Our objective in these experiments is essentially to show that our
conditional gradient splitting is as efficient as other competitors.

The toy problem is the same as the one used by \cite{obozinski10:_multi_task_featur_selec}.  The task is a
binary classification problem in $\R^d$.  Among these $d$ variables,
only $T$ of them define a subspace of $\R^d$ in which classes can be
discriminated.  For these $T$ relevant variables, the two classes
follow a Gaussian pdf with means respectively $\mu$ and $-\mu$ and
covariance matrices randomly drawn from a Wishart distribution.  $\mu$
has been randomly drawn from {$\{-1,+1\}^T$. The other $d-T$}
non-relevant variables follow an \emph{i.i.d} Gaussian probability
distribution with zero mean and unit variance for both classes.  We have  sampled $N$ examples and used $80\%$ of them for training and the rest for
the testing. Before learning, the training set has been normalized to zero mean and unit variance and  test set has been rescaled accordingly. The hyperparameters
$\lambda$ and $\tau$ have been
roughly set so as to maximize the performance on the test set.
We have chosen to initialize  all algorithms with the zero vector.

\begin{figure*}[t]
 \centering
\includegraphics[width=0.6\linewidth]{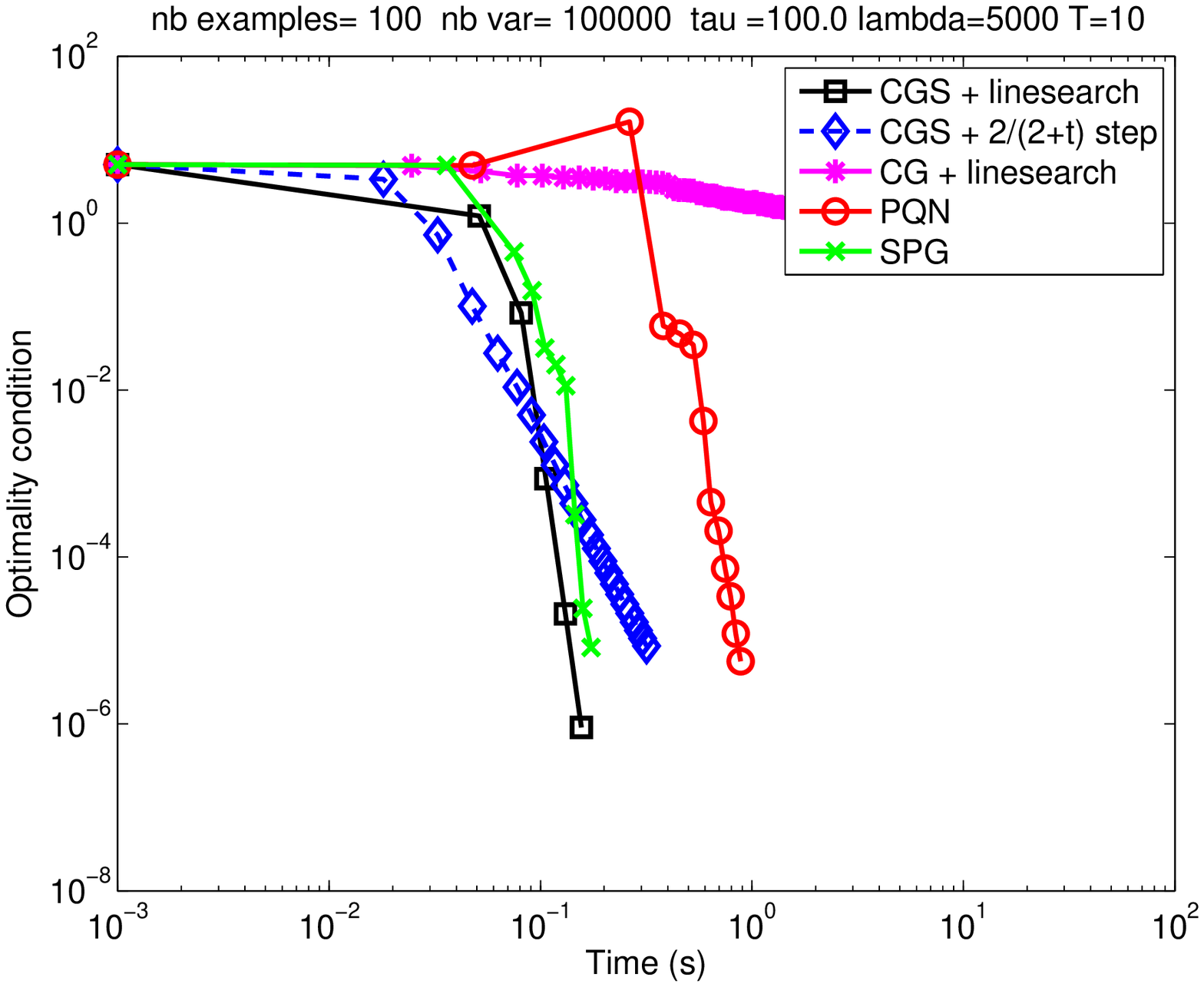}
\includegraphics[width=0.6\linewidth]{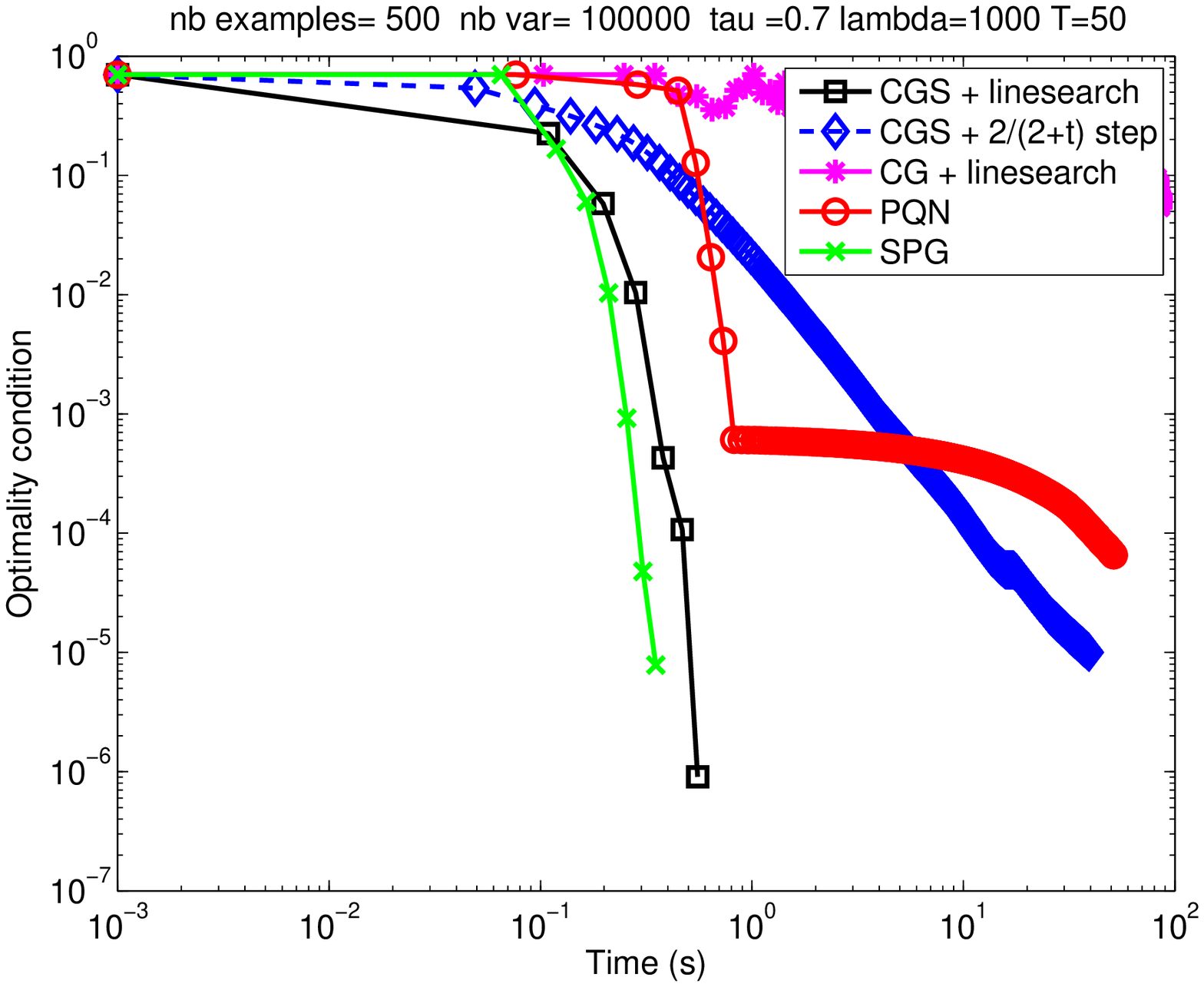}
\includegraphics[width=0.6\linewidth]{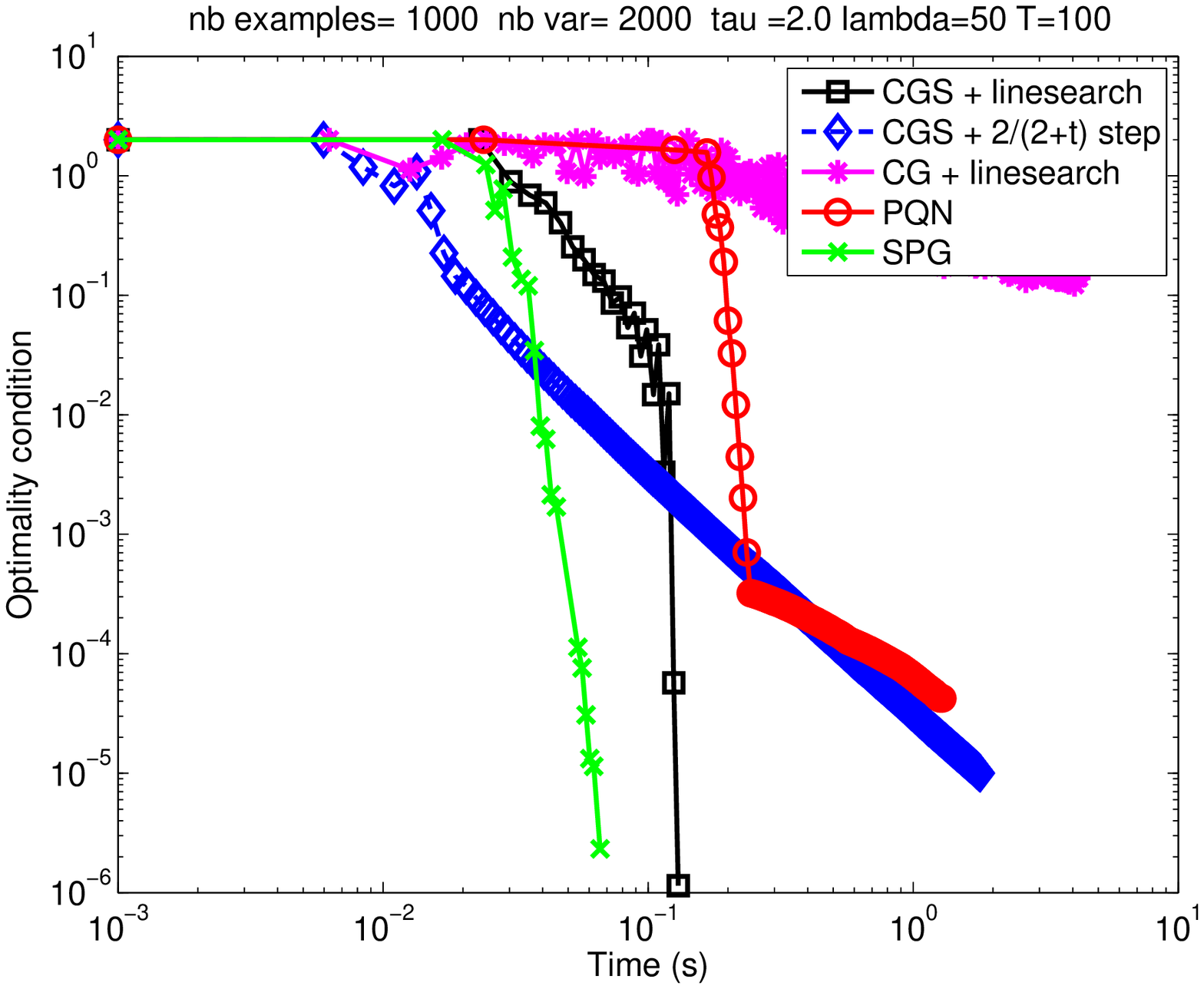}
 \caption{Examples of evolution of the optimality condition for three different
learning setting. (left) highly-sparse and very few examples. (middle) sparse and few examples. (right) sparse with reasonable ratio of examples over variables.}
 \label{fig:elasticnet}
\end{figure*}

 \begin{figure*}[t]
 \centering
\includegraphics[width=0.7\linewidth]{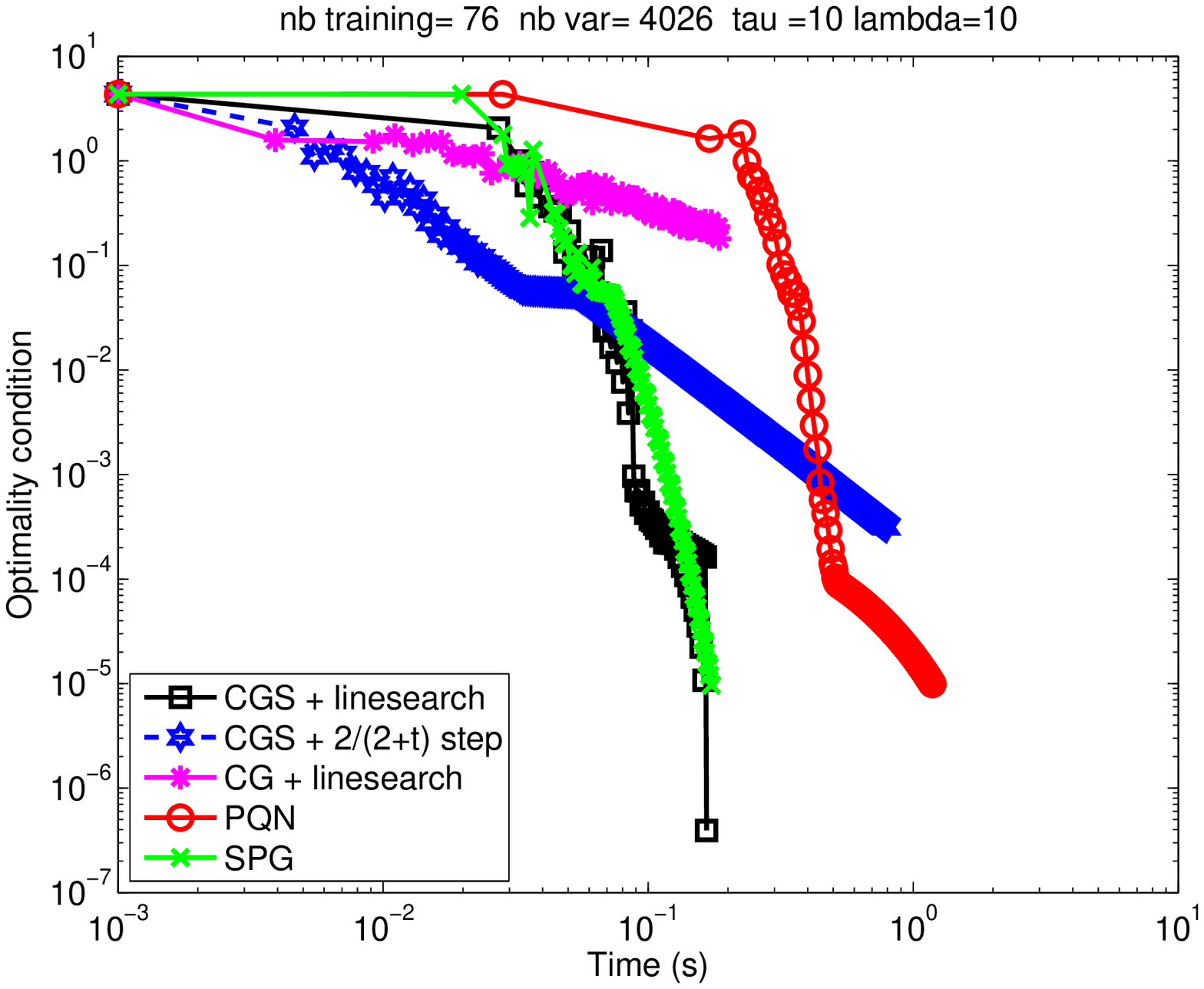}
\includegraphics[width=0.7\linewidth]{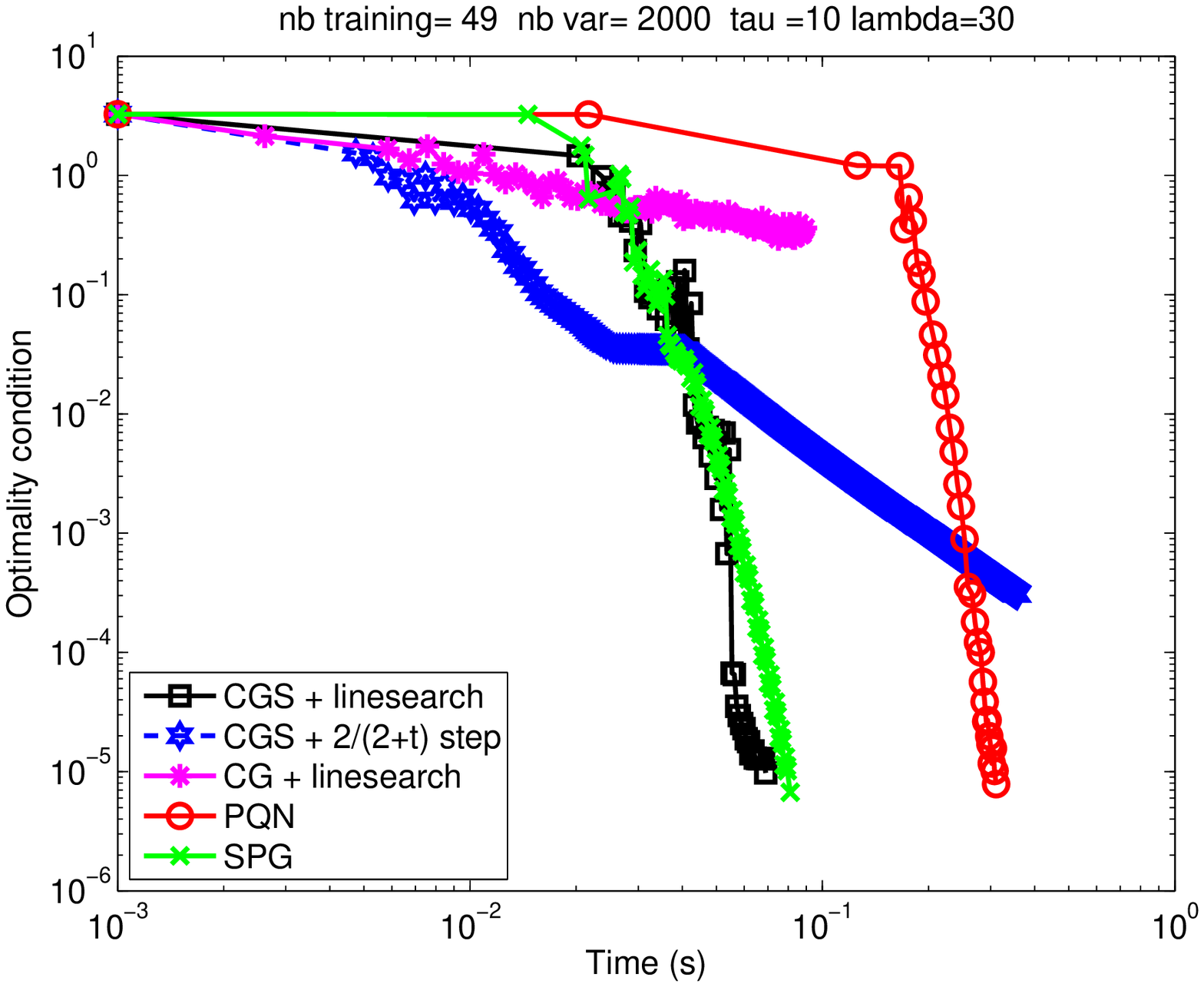}
 \caption{Optimality  conditions evolving curves on the (left) \emph{colon}
and (right) \emph{lymphoma} datasets.}
 \label{fig:real}
\end{figure*}

Figure \ref{fig:elasticnet} presents some examples of how the optimality condition of each method evolves with respect to time for different settings of number
of examples, variables and number of relevant variables. Independently of the
settings, we can note that the conditional gradient algorithm performs very poorly
and is not efficient at all compared to all other algorithms. Compared
to a projected quasi-newton, our conditional gradient splitting algorithm
is far more efficient, and on all the settings it converges faster.
Finally, it appears that our algorithm performs on par with the spectral
projected gradient algorithm, as it is sometimes faster and in other cases
slower. This is a very interesting feature given the simplicity of the
algorithm steps.
In addition, we can note the nice behaviour of our CGS algorithm with
empirical steps which is globally less efficient than CGS with linesearch
and the SPG algorithms but provide better improvements of the optimality
condition in the first iterations.

For illustrating the   algorithm behaviour on real datasets, we have
considered
two bioinformatic problems for which few examples are available while the number
of feature is large: the \emph{colon} and \emph{lymphoma} datasets.
 We have used the same experimental setting as for the toy dataset.
Figure \ref{fig:real} reports typical examples of convergence behaviour.
We can note again than our CGS algorithm is slightly
more efficient than the spectral projected gradient algorithm and more efficient than the limited-memory projected quasi-newton algorithm. Interestingly,
in these real problems, the CGS algorithm with fixed step
is the most efficient one for reaching rough optimality conditions (of
the order of $10^{-2}$).

\subsubsection*{Acknowledgments}
This work was partly funded by the CNRS PEPS Fascido program under the Topase project.

\bibliographystyle{alpha}
\bibliography{optim,DA,TO}

\end{document}